\documentclass{article}

\usepackage[preprint]{nips_2018}

\usepackage[utf8]{inputenc} % allow utf-8 input
\usepackage[T1]{fontenc}    % use 8-bit T1 fonts
\usepackage{hyperref}       % hyperlinks
\usepackage{url}            % simple URL typesetting
\usepackage{booktabs}       % professional-quality tables
\usepackage{amsmath,amsthm,amsfonts,mathtools}       % blackboard math symbols
\usepackage[nice]{nicefrac}       % compact symbols for 1/2, etc.
\usepackage{microtype}      % microtypography

\usepackage[inline]{enumitem}
\usepackage{amsmath,amsthm}
\usepackage{theoremref}
\usepackage{graphicx,xcolor}
\usepackage{tikz}
\usepackage{pgfplots}
\pgfplotsset{compat=1.14}
\usepackage{subcaption}
\usepackage{todonotes}

\newtheorem{definition}{Definition}

\newtheorem{corollary}{Corollary}
\newtheorem{lemma}{Lemma}
\newtheorem{proposition}{Proposition}
\usepackage{adjustbox}

\DeclareMathOperator*{\argmin}{arg\,min}

\title{Function Norms and Regularization in Deep Networks}

\author{
  Amal Rannen Triki\thanks{Amal is also affiliated with Department of Computational Science and Engineering in Yonsei University, South Korea.} \\
  KU Leuven, ESAT-PSI, imec, Belgium\\
  \texttt{amal.rannen@esat.kuleuven.be} \\
  \And
   Maxim Berman \\
  KU Leuven, ESAT-PSI, imec, Belgium\\
  \texttt{maxim.berman@esat.kuleuven.be} \\
  \And
   Matthew B. Blaschko \\
  KU Leuven, ESAT-PSI, imec, Belgium\\
  \texttt{matthew.blaschko@esat.kuleuven.be} \\
}

\begin{document}

\maketitle

\begin{abstract}
Deep neural networks (DNNs) have become increasingly important due to their excellent empirical performance on a wide range of problems.  However, regularization is generally achieved by indirect means, largely due to the complex set of functions defined by a network and the difficulty in measuring function complexity.  There exists no method in the literature for additive regularization based on a norm of the function, as is classically considered in statistical learning theory.  In this work, we propose sampling-based approximations to weighted function norms as regularizers for deep neural networks. We provide, to the best of our knowledge, the first proof in the literature of the NP-hardness of computing function norms of DNNs, motivating the necessity of an approximate approach. We then derive a generalization bound for functions trained with weighted norms and prove that a natural stochastic optimization strategy minimizes the bound. Finally, we empirically validate the improved performance of the proposed regularization strategies for both convex function sets as well as DNNs on real-world classification and image segmentation tasks demonstrating improved performance over weight decay, dropout, and batch normalization. Source code will be released at the time of publication.
\end{abstract}

\section{Introduction}
Regularization is essential in ill-posed problems and to prevent overfitting.  Regularization has traditionally been achieved in machine learning by penalization of a norm of a function or a norm of the parameter vector. In the case of linear functions (e.g.\ Tikhonov regularization \citep{Tikhonov:1963}), penalizing the parameter vector corresponds to a penalization of a function norm as a straightforward result of the Riesz representation theorem \citep{riesz1907espece}.
In the case of reproducing kernel Hilbert space (RKHS) regularization (including splines \citep{wahba1990spline}), this by construction corresponds directly to a function norm regularization \citep{vapnik1998statistical,Scholkopf:2001:LKS:559923}.

In the case of deep neural networks, similar approaches have been applied directly to the parameter vectors, resulting in an approach referred to as weight decay \citep{moody1995simple}.  This, in contrast to the previously mentioned Hilbert space approaches, does not directly penalize a measure of function complexity, such as a norm (\thref{thm:WeightNormNotFunctionNorm}).
Indeed, we show here that any function norm of a DNN with rectified linear unit (ReLU) activation functions \citep{hahnloser2000digital} is NP-hard to compute as a function of its parameter values (Section~\ref{sec:NPhardness}), and it is therefore unreasonable to expect that simple measures, such as weight penalization, would be able to capture appropriate notions of function complexity. 

In this light, it is not surprising that two of the most popular regularization techniques for the non-convex function sets defined by deep networks with fixed topology make use of stochastic perturbations of the function itself (dropout \citep{hinton2012improving,NIPS2013_4878}) or stochastic normalization of the data in a given batch (batch normalization \citep{ioffe2015batch}).  While their algorithmic description is clear, interpreting the regularization behavior of these methods in a risk minimization setting has proven challenging. 
What is clear, however, is that dropout can lead to a non-convex regularization penalty \citep{JMLR:v16:helmbold15a} and therefore does not correspond to a norm of the function.  Other regularization penalties such as path-normalization \citep{NIPS2015_5797} are polynomial time computable and thus also do not correspond to a function norm assuming $P\neq NP$.

Although we show that norm computation is NP-hard, we demonstrate that some norms admit stochastic approximations.  This suggests incorporating penalization by these norms through stochastic gradient descent, thus directly controlling a principled measure of function complexity.  In work developed in parallel to ours, \citet{kawaguchi2017generalization} suggest to penalize function values on the training data based on Rademacher complexity based generalization bounds, but have not provided a link to function norm penalization. 
We also develop a generalization bound, which shows how direct norm penalization controls expected error similarly to their approach.  Furthermore, we observe in our experiments that the sampling procedure we use to stochastically minimize the function norm penalty in our optimization objective empirically leads to better generalization performance (cf.\ Figure~\ref{fig:MNIST_gauss_kawaguchi}).

Different approaches have been applied to explain the capacity of DNNs to generalize well, even though they can use a number of parameters several orders of magnitude larger than the number of training samples. \citet{hardt2015train} analyze stochastic gradient descent (SGD) applied to DNNs using the uniform stability concept introduced by \citet{bousquet2002stability}. However, the stability parameter they show depends on the number of training epochs, which makes the related bound on generalization rather pessimistic  and tends to confirm the importance of early stopping for training DNNs \citep{doi:10.1162/neco.1995.7.2.219}. 
More recently, \citet{zhang2016understanding} have suggested that classical learning theory is incapable of explaining the generalization behavior of deep neural networks.
Indeed, by showing that DNNs are capable of fitting arbitrary sets of random labels, the authors make the point that the expressivity of DNNs is partially data-driven, while the classical analysis of generalization does not take the data into account, but only the function class and the algorithm. 
Nevertheless, learning algorithms, and in particular SGD, seem to have an important role in the generalization ability of DNNs. \citet{keskar2016large} show that using smaller batches results in better generalization. Other works (e.g.\ \citet{hochreiter1997flat}) relate the implicit regularization applied by SGD to the flatness of the minimum to which it converges, but \citet{dinh2017sharp} have shown that sharp minima can also generalize well. 

Previous work concerning the  generalization of DNNs present several contradictory results. 
Taking a step back, it appears that our better understanding of classical learning models -- such as linear functions and kernel methods -- with respect to DNNs comes from the well-defined hypothesis set on which the optimization is performed, and clear measures of the function complexity.  

In this work, we make a step towards bridging the described gap by introducing a new family of regularizers that approximates a proper function norm (Section~\ref{sec:NormBasedReg}).  We demonstrate that this approximation is necessary by, to the best of our knowledge, the first proof in the literature that computing a function norm of DNNs is NP-hard (Section~\ref{sec:NPhardness}). We develop a generalization bound for function norm penalization in Section~\ref{sec:GeneralizationBound} and demonstrate that a straightforward stochastic optimization strategy appropriately minimizes this bound. 

Our experiments reinforce these conclusions by showing that the use of these regularizers lowers the generalization error and that we achieve better performance than other regularization strategies in the small sample regime (Section~\ref{sec:experiments}).

\section{Function norm based regularization}
\label{sec:NormBasedReg}

We consider the supervised training of the weights $W$ of a deep neural network (DNN) given a training set $\mathcal{D} = \{ (x_i,y_i)\} \in (\mathcal{X}\times\mathcal{Y})^n$, where $\mathcal{X} \subseteq \mathbb{R}^d$ is the input space and $\mathcal{Y}$ the output space. 
Let $f: \mathcal{X} \rightarrow \mathcal{\tilde{Y}} \subseteq\mathbb{R}^s$ be the function encoded by the neural network. 
The prediction of the network on an $x \in \mathcal{X}$ is generally given by $D \circ f(x) \in \mathcal{Y}$, where $D$ is a decision function. 
For instance, in the case of a classification network, 
$f$ gives the unnormalized scores of the network.
During training, the loss function $\ell$ penalizes the outputs $f(x)$ given the ground truth label $y$, and
%We consider the usual setting where the prediction corresponding to $f(x)$ 
%In a supervised training setting, the weights $W$ are optimized given a training set 
%In the typical cases that we consider, the prediction of the network $f(x) \in \mathcal{Y}$
% the DNN defines for every parameterization $W$ a function $f$ between $\mathcal{X}$ and a score space of the same dimension of $\mathcal{Y}$ and that is mapped itself to $\mathcal{Y}$ through a decision function.
we aim to minimize the risk% of a function $f$
\begin{equation}\label{eq:Risk}
\mathcal{R}(f) = \int\! \ell(f(x),y)\, \mathrm{d}P(x,y),
\end{equation}
where $P$ is the underlying joint distribution of the input-output space. As this distribution is generally inaccessible, empirical risk minimization approximates the risk integral~\eqref{eq:Risk} by
% using empirical risk minimization, where this risk integral is approximated by
% Using empirical risk minimization Training the network aims to minimize a finite sample approximation of the risk integral, the empirical risk
% To this effect finite samples of the risk integral are 
\begin{equation}\label{eq:empRisk}
\hat{\mathcal{R}}(f) = \frac{1}{n} \sum_{i=1}^n \ell (f(x_i),y_i),
\end{equation}
where the elements from the dataset $\mathcal{D}$ are supposed to be %$\{ (x_i,y_i)\} \in (\mathcal{X}\times\mathcal{Y})^n$ are 
i.i.d.\ samples drawn from $P(x,y)$.

When the number of samples $n$ is large, the empirical risk~\eqref{eq:empRisk} is a good approximation of the risk~\eqref{eq:Risk}. 
In the small-sample regime, however, better control of the generalization error can be achieved by adding a regularization term to the objective. 
%For a better control over the generalization error, especially in the small sample regime, we usually introduce a regulation term to the objective.  
In the statistical learning theory literature, this is most typically achieved through an additive penalty \citep{vapnik1998statistical,Murphy2012MLP}
\begin{equation}
\argmin_f \hat{\mathcal{R}}(f) + \lambda \Omega(f),
\label{eq:regObj}
\end{equation}
where $\Omega$ is a measure of function complexity. %Our main contribution is to use 
%$\Omega(f_W) \approx \|f_W\|_*$ where $\|.\|_*$ is a proper function norm.
The regularization %aims at modifying 
biases the objective towards ``simpler'' candidates in the model space. %in order to find a ``simple'' candidate for the problem at hand. 

In machine learning, using the norm of the learned mapping appears as a natural choice to control its complexity. 
This choice limits the hypothesis set to a ball in a certain topological set depending on the properties of the problem. In an RKHS, the natural regularizer is a function of the Hilbert space norm: for the space $\mathcal{H}$ induced by a kernel $K$, $\|f\|_\mathcal{H}^2 = \langle f, f\rangle_\mathcal{H}$. Several results showed that the use of such a regularizer results in a control of the generalization error \citep{girosi1990networks,wahba1990spline,bousquet2002stability}.
In the context of function estimation, for example using splines, it is customary to use the norm of the approximation function or its derivative in order to obtain a regression that generalizes better \citep{wahba2000splines}.

However, for neural networks, defining the best prior for regularization is less obvious. The topology of the function set represented by a neural network is still fairly unknown, which complicates the definition of a proper complexity measure. 

\begin{lemma}\thlabel{thm:WeightNormNotFunctionNorm}
The norm of the weights of a neural network, used for regularization in e.g.\ weight decay, is not a proper function norm. 
\end{lemma}
It is easy to see that different weights $W$ can encode the same function $f$, for instance by permuting neurons or rescaling different layers. Therefore, the norm of the weights is not even a function of $f$ encoded by those weights. %More precisely, scaling the weights by $\alpha > 0$ would scale their norm by the same factor while scaling the output by $\alpha^h$ where $h$ is the depth of the network, in the case of a network with positively homogeneous activation functions (e.  g. ReLU). 
Moreover, in the case of a network with ReLU activations, it can easily be seen that the norm of the weights does not have the same homogeneity degree as the output of the function, which induces optimization issues, as detailed in~\cite{haeffele2017global}
% This can induce optimization issues, as shown in. 

%~\cite{haeffele2017global} shows that weight decay 
% It is easy to see that a given $f_W$ can map to different weights, and thus the weight norm is not even a function of $f_W$. 

Nevertheless, if the activation functions are continuous, any function encoded by a network is in the space of continuous functions. Moreover, supposing the input domain $\mathcal{X}$ is compact, the network function has a finite $L_q$-norm. 

\begin{definition}[$L_q$-norm]
Given a measure $\mu$, the function $L_q$-norm for $q \in \left[1,\infty \right] $ is defined as
\begin{equation}
\|f\|_q = \left(\int\! \|f(x)\|_q^q \ \mathrm{d}\mu (x)\right)^{\frac{1}{q}} ,
\label{eq:qNorm}
\end{equation}
where the inner norm represents the $q$-norm of the output space.
\end{definition}
%In this paper, we consider the feasibility of regularizing the training of the network with this type of function norms.
 %Additionally, adding this type of regularizer has the virtue of ``convexifying'' the objective in the function space.

In the sequel, we will focus on the special case of $L_2$. This function space has attractive properties, being a Hilbert space. Note that in an RKHS, controlling the $L_2$-norm can also control the RKHS norm under some assumptions. When the kernel has a finite norm, the inclusion mapping between the RKHS and $L_2$ is continuous and injective, and constraining the function to be in a ball in one space constrains it similarly in the other  \citep[Chapter~4]{mendelson2010regularization,steinwart2008support}.%, the $L_2$-norm of the function is the first regularizer that we consider. Moreover, as in general smoothness of the function induces its robustness to input perturbations, we consider the possibility of penalizing the gradient of the function with respect to its input. The function is then treated as a member of the Sobolev space $H_2$.
% \begin{definition}[Sobolev $H_2$ norm \citep{MazjaVladimirG1985Ss}]
% \begin{equation}
% \|f\|_{H_2} = \left(\|f\|_2^2 + \|\nabla_x f\|_2^2\right)^{1/2}
% \end{equation}
% As in the definition of Sobolev spaces as Banach spaces, the gradient here is to be understood in the weak sense, meaning that this definition applies for functions that are differentiable almost everywhere as it is the case for neural networks with ReLU activations.
% \end{definition}
% The Sobolev norm, being a combination of the $L_2$ norm and a penalization of the gradient, inherits the favorable properties of the $L_2$ norm while additionally controlling the function variation.

However, because of the high dimensionality of neural network function spaces, the optimization of function norms %(Eq.~\eqref{eq:qNorm}) 
is not an easy task. 
Indeed, the exact computation of any of these function norms is NP-hard, as we show in the following section.

\section{NP-hardness of function norm computation}\label{sec:NPhardness}

%Let us go back to the exact $L_2$ function norm (Eq.~\eqref{eq:exactFnNorm}).

\begin{proposition} 
For $f$ defined by a deep neural network (of depth greater or equal to 4) with ReLU activation functions, the computation of any function norm
$ \| f \| \in \mathbb{R}^+$ 
from the weights of a network is NP-hard. 
\thlabel{proposition:NPhardness}
\end{proposition}

We prove this statement by a linear time reduction of the classic NP-complete problem of Boolean 3-satisfiability~\cite{Cook1971CTP} to the computation of the norm of a particular network with ReLU activation functions.  Furthermore, we can construct this network such that it always has finite $L_2$-norm. 
% \input{proof_np}
% The sketch of the proof is as follows

%Proving \thref{proposition:NPhardness} is equivalent to proving the following:

\begin{lemma}\thlabel{thm:NormNPhardConstruction}
Given a Boolean expression $\mathcal{B}$ with $p$ variables 
in conjunctive normal form in which each clause is composed of three literals, we can construct, in time polynomial in the size of the predicate, a network of depth~$4$  and realizing a continuous function $f:  \mathbb{R}^p \rightarrow \mathbb{R}$ that has non-zero $L_2$ norm if and only if the predicate is satisfiable. 
\end{lemma}
\begin{proof}
See Supplementary Material for a construction of this network.%, of polynomial size in the size of $\mathcal{B}$.
\end{proof}

\begin{corollary}\thlabel{th:NPhardAllNorms}
Although not all norms are equivalent in the space of continuous functions, \thref{thm:NormNPhardConstruction} implies that any function norm for a network of depth $\geq 4$ must be NP-hard since for all norms $\|f\|=0 \iff f=\mathbf{0}$.
\end{corollary}

% Alternatively to the untractable $L_2$ norm, we may consider more general families 
%As the exact computation of an $L_2$ function norm is not possible, we 
%might consider a sampling based approximation of the form
% \begin{equation}
% \|f\|_{2} \simeq \frac{1}{m}\sum_{j=1}^m \| f(x_j)\|_2^2
% \end{equation}
% where the $x_j$ are sampled i.i.d from the input distribution of $\mathcal{X}$. 
% As the exact computation of an exact $L_2$ norm is not possible, a sampling based approach is necessary.
% Sampling uniformly in the input space, however, is not a valid option because of its high dimensionality in practice for DNNs. 
%The distribution of the input space is however generally unknown
This shows that the exact computation of any $L_2$ function norm is intractable. 
However, assuming the measure $\mu$ in the definition of the norm~\eqref{eq:qNorm} is a probability measure $Q$, the function norm can be written as $\|f\|_{2,Q} = \mathbb{E}_{z \sim Q}\left[ \|f(z)\|_2^2 \right]^{\nicefrac{1}{2}} $. 
%and approximated with samples 
%However, if we have access to a sampling distribution $Q$, we may  for any probability distribution $Q$, then  is a proper function norm. 
Moreover, assuming we have access to i.i.d samples $z_j \sim Q$, %can sample $\{, j = 1 \ldots m\}$ from $Q$, 
this weighted $L_2$-function norm can be approximated by
\begin{equation}\label{eq:sampleMean}
\left(\frac{1}{m} \sum_{i=1}^m \|f(z_i)\|_2^2\right)^{\frac{1}{2}}.
\end{equation}
For samples outside the training set, empirical estimates of the squared weighted $L_2$-function norm are $U$-statistics of order 1, and have an asymptotic Gaussian distribution to which finite sample 
estimates converge quickly as $\mathcal{O}(m^{-1/2})$ \citep{lee1990u}.
In the next section, we demonstrate sufficient conditions under which control of $\|f\|_{2,Q}^2$ results in better control of the generalization error.  

\section{Generalization bound and optimization}\label{sec:GeneralizationBound}

In this section, rather than the regularized objective of the form of Equation~\eqref{eq:regObj}, we consider an equivalent constrained optimization setting. The idea about controlling an $L_2$ type of norm is to attract the output of the function towards 0, effectively limiting the confidence of the network, and thus the values of the loss function. Classical bounds on the generalization show the virtue of a bounded loss. As we are approximating a norm with respect to a sampling distribution, this bound on the function values can only be probably approximately correct, and will depend on the statistics of the norm of the outputs--namely the mean (i.e.\ the $L_{2,Q}$-norm) and the variance, as detailed by the following proposition:
\begin{proposition}\thlabel{th:generalization}
When the number of samples $n$ is small, and if we suppose $\mathcal{Y}$ bounded, and $\ell$ Lipschitz-continuous, solving the problem
\begin{align}
f_* &= \argmin_f \hat{\mathcal{R}}(f), %\\
&\text{s.t.}\quad\|f_*\|_{2,Q}^2 \le A \quad\text{and}\quad  \operatorname{var}_{z \sim Q}(\|f_*(z)\|_2^2) \le B^2 \label{eq:Conditions}
%; \  \sigma_{fQ} = \operatorname{var}_{Q}(\|f(x)\|_2^2)^\frac{1}{2} \le \eta;
\end{align}
effectively reduces the complexity of the hypothesis set, and the bounds $A$ and $B$ on the weighted $L_2$-norm and the standard deviation control the generalization error, provided that  $\mathcal{D}_P(P\|Q) =  \int\! \frac{P(\nu)}{Q(\nu)} P(\nu)\, \mathrm{d}\nu$ is small, where $P$ the marginal input distribution and $Q$ the sampling distribution.\footnote{We note that $\mathcal{D}_{P}(P\|Q) -1$ is the $\chi^2$-divergence between $P$ and $Q$ and is minimized when $P=Q$.} 
% \begin{equation}
% \exists \gamma > 0,  , \operatorname{D_{KL}}(P||Q) \le \gamma
% \end{equation}
% where  is the Kullback-Leibler divergence between $P$ and $Q$ .
Specifically, the following generalization bound holds with probability at least $(1-\delta)^2$:
\begin{equation}
\mathcal{R}(f_*) \le \hat{\mathcal{R}}(f_*)  + %\tilde{L}(A,B,  \mathcal{D}_P(P\|Q))
\left(K\left[\frac{ (A+B)^\frac{1}{2} \mathcal{D}_P(P\|Q)^\frac{1}{4}  }{\sqrt{\delta}} + A^\frac{1}{2} \mathcal{D}_P(P\|Q)^\frac{1}{2}\right] + C\right)
\sqrt{\frac{2\ln \frac{2}{\delta}}{N}}. 
\end{equation}
\end{proposition}
The proof can be found in the supplementary material, Appendix~\ref{sec:AppendixProofGeneralizationBound}.

\subsection{Practical optimization}
In practice, we try to get close to the ideal conditions of~\thref{th:generalization}. The Lipschitz continuity of the loss and the boundedness of $\mathcal{Y}$ hold in most of the common situations. Therefore, three conditions require attention:
(i) the norm $\|f_*\|_{2,Q}$;
(ii) the standard deviation of $\|f(z)\|_2^2$ for $z \sim Q$;
(iii) the relation between the sampling distribution and the marginal distribution.
Even if we can generate samples from the distribution $Q$, at each step of training, only a batch of limited size can be presented to the network. Nevertheless, controlling the sample mean of a different batch at each iteration can be sufficient to attract all the observed realizations of the output towards 0, and therefore simultaneously bound both the expected value and the standard deviation. 

\begin{proposition} \thlabel{th:boundSampleMean}
If for a fixed $m$, for all samples $\{z_i\sim Q\}$ of size $m$:
\begin{equation}\label{eq:boundSampleMean}
\frac{1}{m} \sum \|f(z_i)\|_2^2 \le A, 
\end{equation}
then $\|f\|_{2,Q}^2$ and $var_{z\sim Q}(\|f(z)\|_2^2)$ are also bounded.
\end{proposition}

\begin{proof}
If for any sample $\{z_i\sim Q\}$ of size $m$, the condition~\eqref{eq:boundSampleMean} holds, then: 
\begin{equation}
\forall z_i\sim Q, \|f(z_i)\|_2^2 \le mA 
\end{equation}
and
\begin{equation}
\mathbb{E}_{z\sim Q} [\|f(z)\|_2^2] \le mA ; \operatorname{var}_{z\sim Q}(\|f(z)\|_2^2) \le \mathbb{E}_{z\sim Q} [\|f(z)\|_2^4] \le m^2A^2 .
\end{equation}
\end{proof}
While training, in order to satisfy the two first conditions, we use small batches to estimate the function norm with the expression~\eqref{eq:sampleMean}. When possible, a new sample is generated at each iteration in order to approach the condition in \thref{th:boundSampleMean}. Concerning the condition on the relation between the two distributions, three possibilities where considered in our experiments: (i) using unlabeled data that are not used for training, (ii) generating from a Gaussian distribution that have mean and variance related to training data statistics, and (iii) optimizing a generative model,  e. g. a variational autoencoder \citep{kingma2013auto} on the training set. In the first case, the sampling is done with respect to the data marginal distribution, in which case the derived generalization bound is the tightest. However, in this case, we can use only a limited number of samples, and our control on the function norm can be loose because of the estimation error. In the second and third case, it is possible to generate as many samples as needed to estimate the norm. The Gaussian distribution satisfy the boundedness of $D_P(P\|Q))$, but does not take into account the spatial data structure.  The variational autoencoder, in the contrary, captures the spatial properties of the data, but suffers from mode collapse. In order to alleviate the effect of having a tighter distribution than the data, we use an enlarged Gaussian distribution in the latent space when generating the samples from the trained autoencoder.

\section{Experiments and results} \label{sec:experiments}
To test the proposed regularizer, we consider three different settings:
\begin{enumerate*}[label=(\roman*)]
\item A classification task with kernelized logistic regression,
for which control of the weighted $L_2$ norm theoretically controls the RKHS norm, and should therefore result in accuracy similar to that achieved by standard RKHS regularization;
%where the hypothesis set is convex, and our derived stability and generalization bound holds without any additional assumptions on a network topology or degree of regularization;
\item A classification task with DNNs;
% \item A medical image segmentation task with DNNs.
\item A semantic image segmentation task with DNNs.
\end{enumerate*}

\subsection{Oxford Flowers classification with kernelized logistic regression}

\begin{figure}
\centering
\includegraphics[width=0.3\columnwidth]{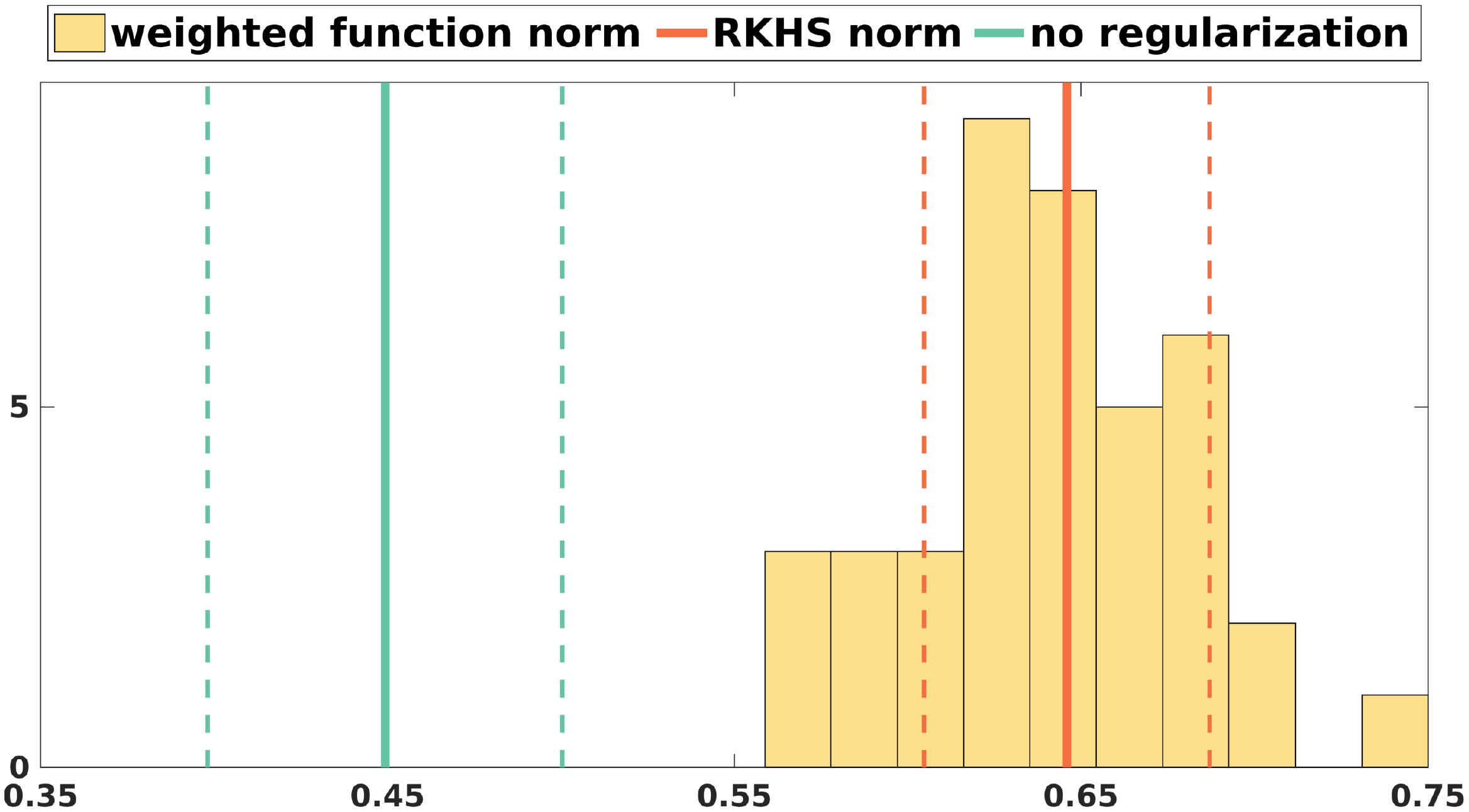}
\caption{Histogram of accuracies with weighted function norm on the Oxford Flowers dataset over 10 trials with 4 different regularization sample sizes, compared to the mean and standard deviation of RKHS norm performance, and the mean and standard deviation of the accuracy obtained without regularization. }
\label{fig:RKHSvsOurs}
\end{figure}

In Sec.~\ref{sec:NormBasedReg}, we state that according to~\cite{steinwart2008support,mendelson2010regularization}, the $L_2$-norm regularization should result in a control over the RKHS norm.  The following experiment shows that both norms have similar behavior on the test data.

\paragraph{Data and Kernel}
For this experiment we consider the 17 classes Oxford Flower Dataset, composed of 80 images per class, and precomputed kernels that have been shown to give good performance on a classification task~\cite{Nilsback06,Nilsback08}.  We have taken the mean of Gaussian kernels as described in~\cite{GehlerN09}.

\paragraph{Settings}
To test the effect of the regularization, we train the logistic regression on a subset of 10\% of the data, and test on 20\% of the samples. The remaining 70\% are used as potential samples for regularization. For both regularizers, the regularization parameter is selected by a 3-fold cross validation. For the weighted norm regularization, we used a 4 different sample sizes ranging from 20\% to 70\% of the data as this results in a favorable balance between controlling the terms in Eq.~\eqref{eq:Conditions} (cf.\ \thref{th:boundSampleMean}). This procedure is repeated on 10 different splits of the data for a better estimate. The optimization is performed by quasi-Newton gradient descent, which is guaranteed to converge due to the convexity of the objective.  
\paragraph{Results}
Figure~\ref{fig:RKHSvsOurs} shows the means and standard deviations of the accuracy on the test set obtained without regularization, and with regularization using the RKHS norm, along with the histogram of accuracies obtained with the weighted norm regularization with the different sample sizes and across the ten trials. This figure demonstrates the equivalent effect of both regularizer, as expected with the stability properties induced by both norms. 

The use of the weighted function norm is more useful for DNNs, where very few other direct function complexity control is known to be polynomial. The next two experiments show the efficiency of our regularizer when compared to other regularization strategies: Weight decay~\cite{moody1995simple}, dropout~\cite{hinton2012improving} and batch normalization~\cite{ioffe2015batch}. 

\subsection{MNIST classification}
\paragraph{Data and Model } In order to test the performance of the tested regularization strategies, we consider only small subsets of 100 samples of the MNIST dataset for training. The tests are conducted on 10,000 samples. We consider the LeNet architecture~\cite{lecun1995comparison}, with various combinations of weight decay, dropout, batch normalization, and  weighted function norms (Figure~\ref{fig:MNIST}).

\paragraph{Settings } We train the model on 10 different random subsets of 100 samples. For the norm estimation, we consider both generating from Gaussian distributions and from a VAE trained for each of the subsets. The VAEs used for this experiment are composed of 2 hidden layers  as in~\cite{kingma2013auto}. 
More details about the training and sampling are given in the supplementary material. For each batch, a new sample is generated for the function norm estimation. SGD is performed using ADAM~\cite{kingma2014adam} for the training of the VAE and plain SGD with momentum is used for the main model. The obtained models are applied to the test set, and classification error curves are averaged over the 10 trials. The regularization parameter is set to 0.01 for all experiments.

\begin{figure}[ht]
\centering
\begin{subfigure}{0.48\textwidth}
\centering
\includegraphics[width=\columnwidth]{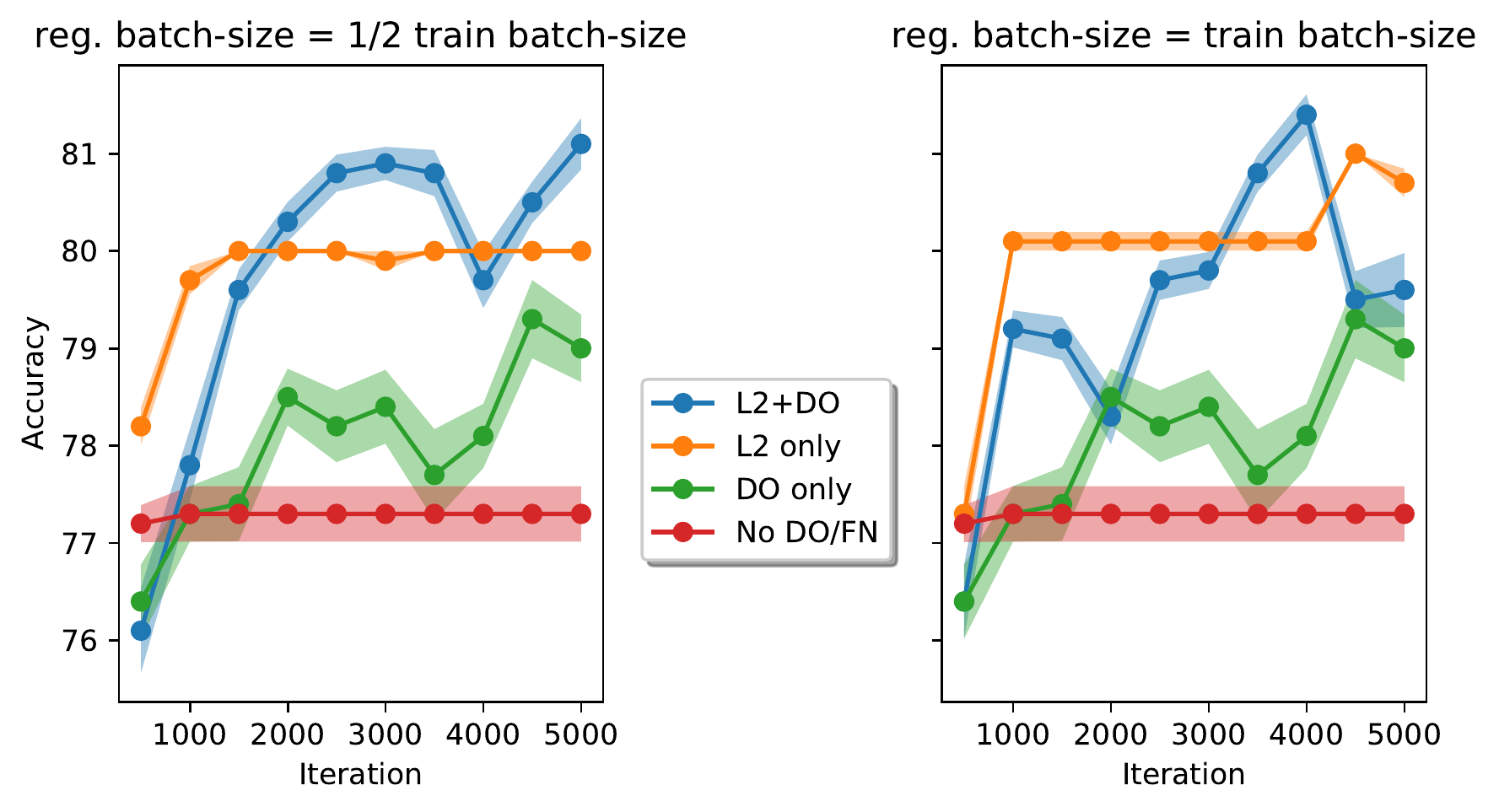}
\caption{Compare to Dropout}\label{fig:MNIST_vae_do}
\end{subfigure} \hfill \begin{subfigure}{0.48\textwidth}
\centering
\includegraphics[width=\columnwidth]{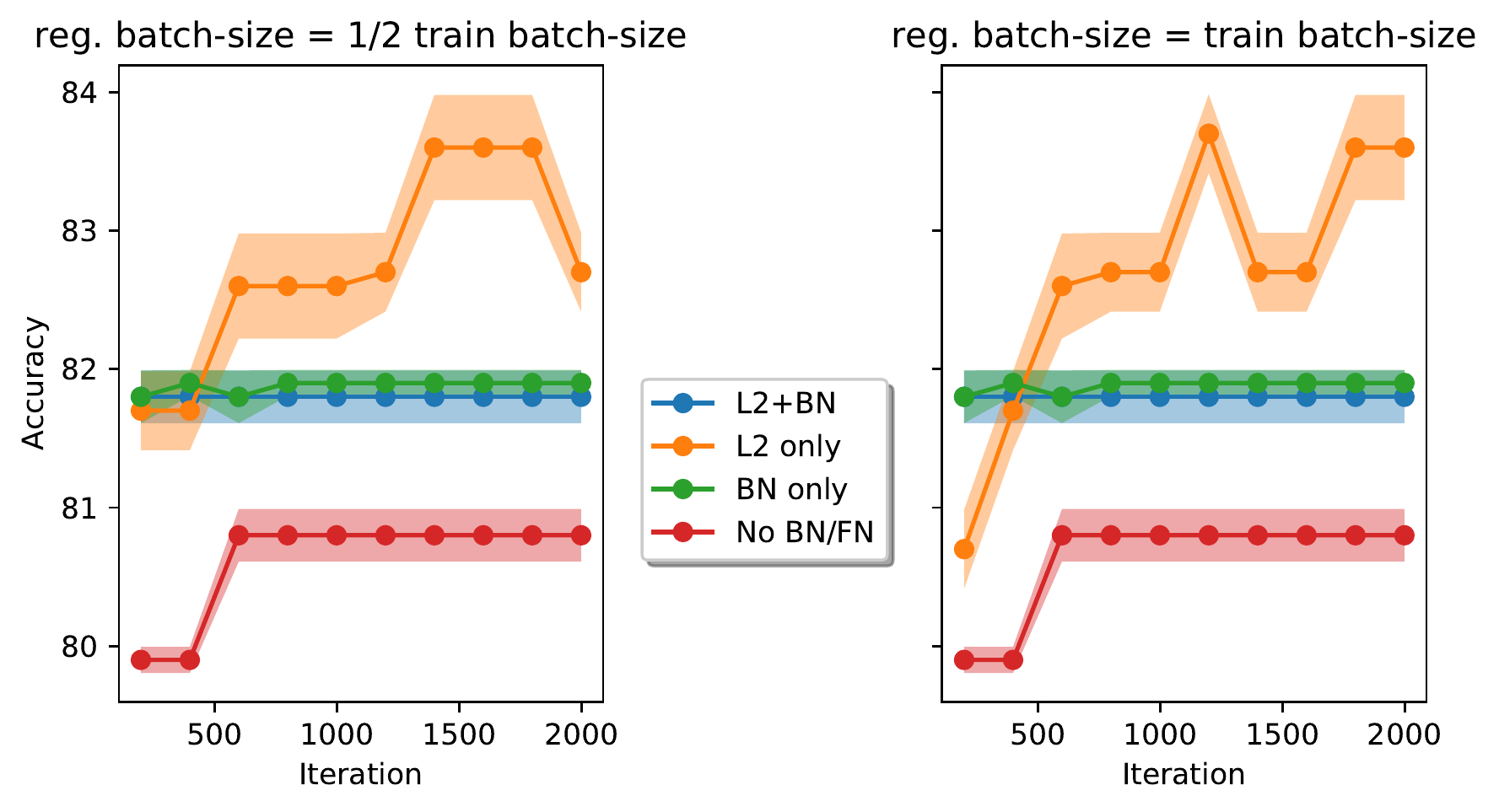}
\caption{Compare to batch-normalization}
\label{fig:MNIST_vae_bn}
\end{subfigure} 
\caption{Performance of $L_2$ norm using VAE for generation, compared to batch-normalization and dropout. All the models use weight decay. The size of the regularization batch is half of the training batch in the left and equal to the training batch in the right of each of the subfigures.}
\label{fig:MNIST}
\end{figure}

% \begin{figure}
% \centering
% \includegraphics[width=0.7\textwidth]{compare_dropout_gauss}
% \caption{Gaussians with different mean and variance}
% \end{figure}

% \begin{figure}
% \centering
% \includegraphics[width=0.4\textwidth]{compare_l2_DARC}
% \caption{Comparison to~\cite{kawaguchi2017generalization}}
% \end{figure}

\begin{figure}
\begin{subfigure}{0.6\linewidth}
\centering
\includegraphics[width=\textwidth]{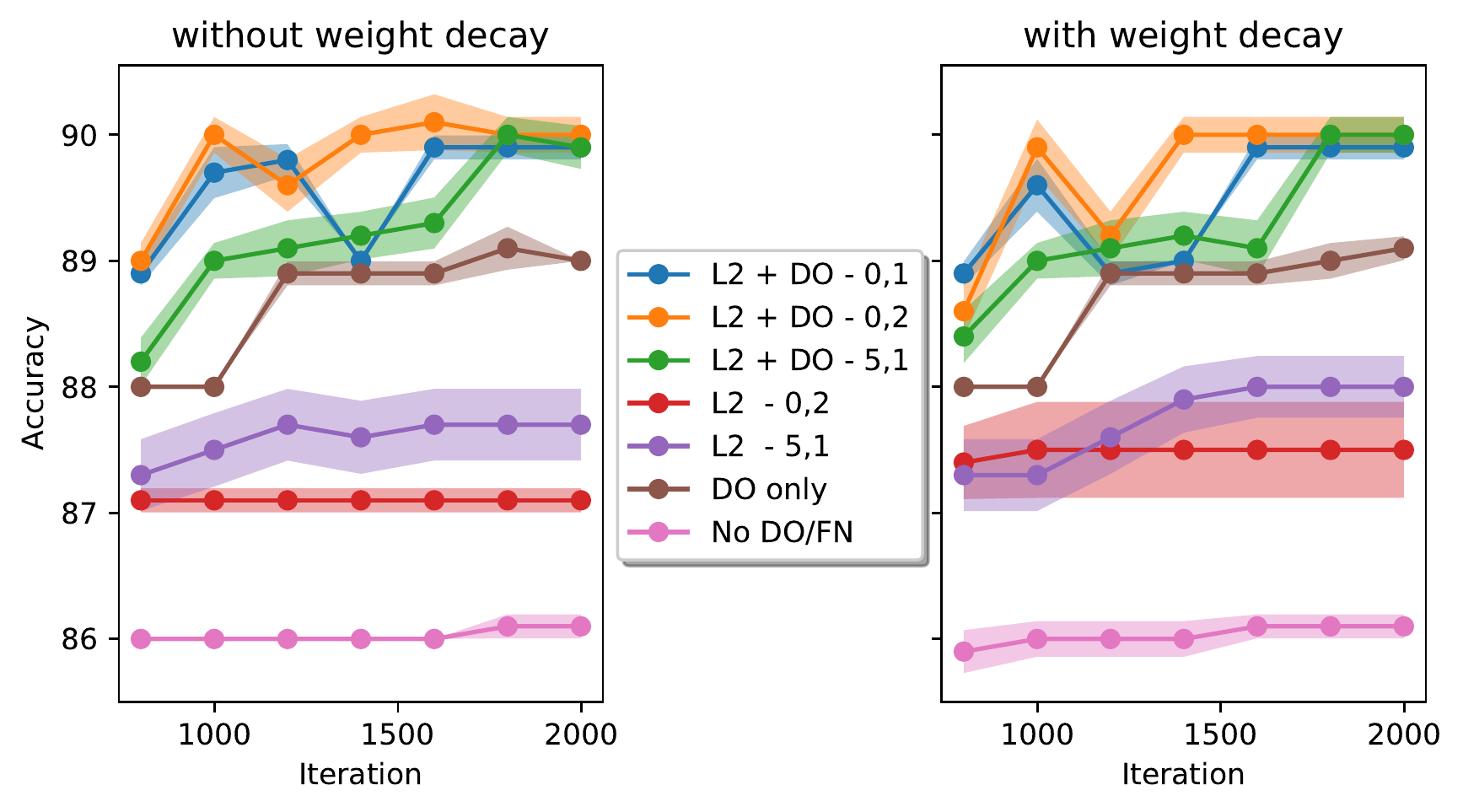}
\caption{Gaussians with different mean and variance\label{fig:MNIST_gauss}}
\end{subfigure}%
\begin{subfigure}{0.4\linewidth}
\centering
\includegraphics[width=\textwidth]{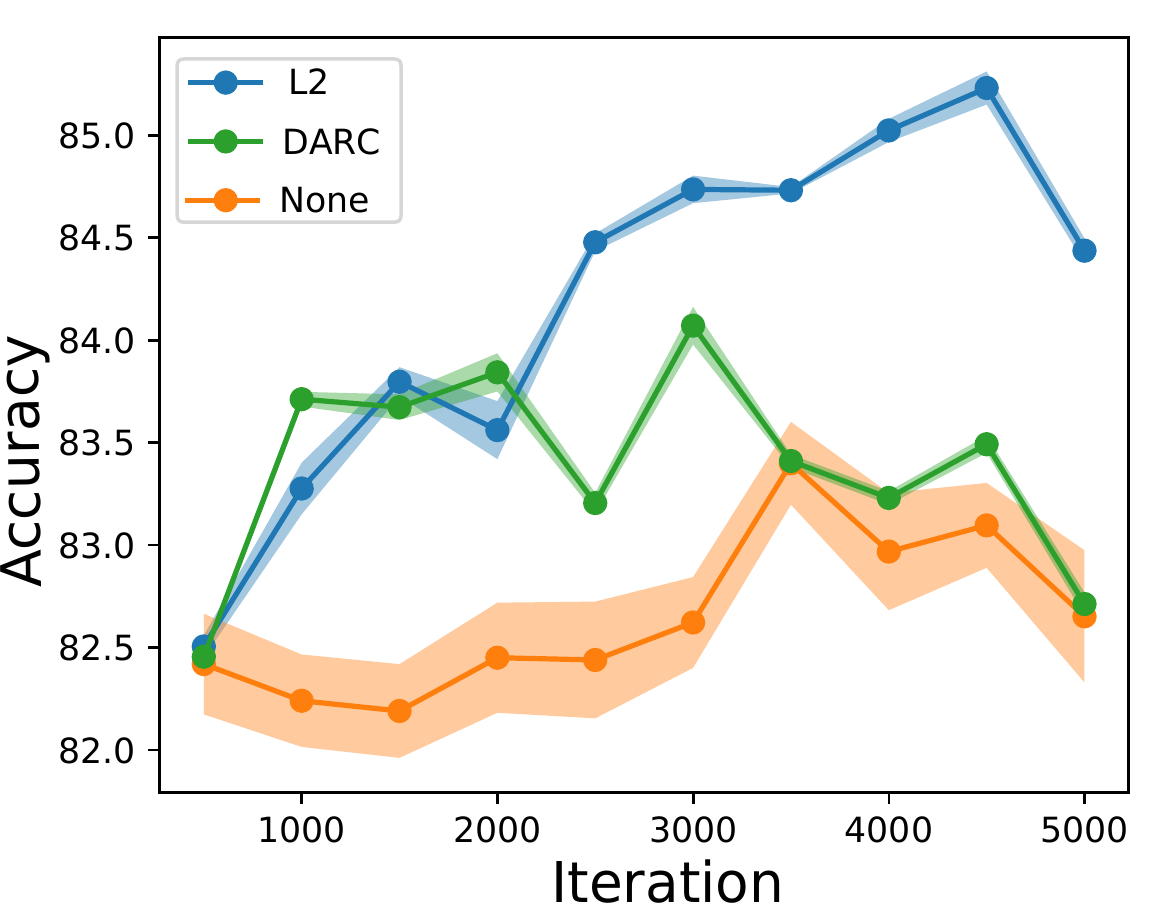}
\caption{Comparison to~\cite{kawaguchi2017generalization}\label{fig:MNIST_kawaguchi}}
\end{subfigure}
\caption{Several experiments with MNIST using Dropout, weight decay and different regularizers. In the left, a function norm regularization with samples generated from Gaussian distribution is used. The mean and variance are indicated in the legend of the figure. In the right, we compare function norm regularization with VAE to the regularizer introduced in~\cite{kawaguchi2017generalization}.\label{fig:MNIST_gauss_kawaguchi}}
\end{figure}

\paragraph{Results}
Figure~\ref{fig:MNIST} displays the averaged curves and error bars for two different architectures for MNIST. Figure~\ref{fig:MNIST_vae_do} compares the effect of the function norm to dropout and weight decay. Figure~\ref{fig:MNIST_vae_bn} compares the effect of the function norm to dropout and weight decay. Two different sizes of regularization batches are used, in order to test the effect of this parameter. It appears that a higher batch size can reach higher performances but seems to have a higher variance, while the smaller batch size shows more stability with comparable performance at convergence. These experiments show a better performance of our regularization when compared with dropout and batch normalization. Combining our regularization with dropout seems to increase the performance even more, but batch-normalization seems to annihilate the effect of the $L_2$ norm. 

Figure~\ref{fig:MNIST_gauss_kawaguchi} displays the averaged curves and error bars for various experiments using dropout. Figure~\ref{fig:MNIST_gauss} shows the results using Gaussian distributions for generation  of the regularization samples. Using Gaussians with mean 5 and  variance 2, and mean 10 and variance 1 caused the training to diverge and yielded only random performance. Figure~\ref{fig:MNIST_kawaguchi} shows that our method outperforms the regularizer proposed in~\cite{kawaguchi2017generalization}. 

Note that each of the experiments use a different set of randomly generated subsets for training. However, the curves in each individual figure use the same data.

In the next experiment, we show that weighted function norm regularization can improve performance over batch normalization on a real-world image segmentation task.

\subsection{Regularized training of ENet}
We consider the training of ENet~\cite{paszke2016enet}, a network architecture designed for fast image segmentation, on the Cityscapes dataset~\cite{cordts2016cityscapes}. 
As regularization plays a more significant role in the low-data regime, we consider a fixed random subset of $N=500$ images of the training set of Cityscapes as an alternative to the full $2975$ training images. 
We compare train ENet similarly to the author's original optimization settings, in a two-stage training of the encoder and the encoder + decoder part of the architecture, using weighted cross-entropy loss. 
We use Adam a base learning rate of $2.5\cdot 10^{-4}$ with a polynomially decaying learning rate schedule and $90000$ batches of size $10$ for both training stages. 
We found the validation performance of the model trained under these settings with all images to be $60.77\%$ mean IoU; this performance is reduced to $47.15\%$ when training only on the subset. 
% As generative models for such high-resolution inputs is largely an open question, we use unlabeled samples on the  instead of generated samples. 

We use our proposed weighted function norm regularization using unlabeled samples taken from the $20 000$ images of the ``coarse'' training set of Cityscapes, disjoint from the training set. 
Figure~\ref{fig:enet_curve} shows the evolution of the validation accuracy during training. We see that the added regularization leads to a higher performance on the validation set. 
%The resulting segmentation masks are lightly affected by the regularization; some surfaces appear smoother, as the pavement in the example of 
Figure~\ref{fig:enet_outputs} shows a segmentation output with higher performance after adding the regularization. 
\begin{figure}
\centering
\includegraphics[width=.5\textwidth]{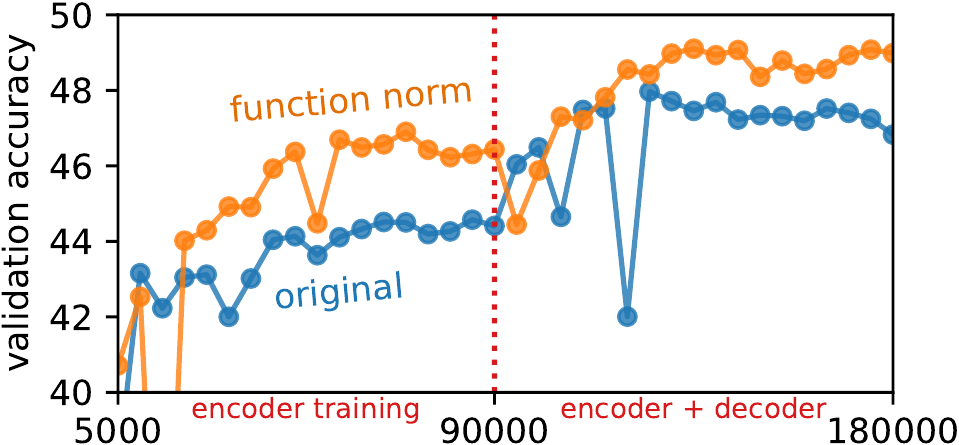}
\caption{Evolution of the validation accuracy of ENet during training with the network's original regularization settings, and with added weighted function norm regularization.~\label{fig:enet_curve}}
\end{figure}
\begin{figure}
\begin{subfigure}{0.33\linewidth}
\includegraphics[width=\textwidth]{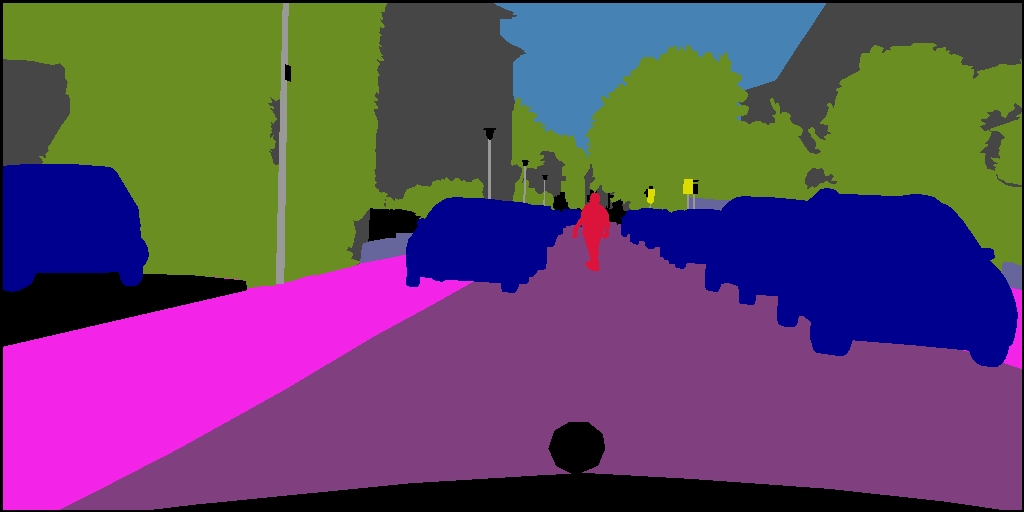}
\caption{ground truth}
\end{subfigure}\hfill\begin{subfigure}{0.33\linewidth}
\includegraphics[width=\textwidth]{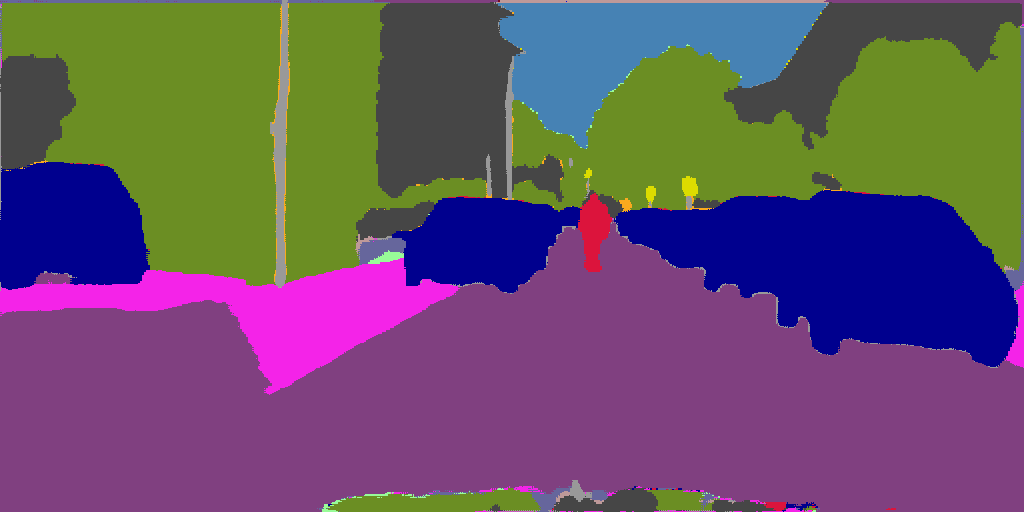}
\caption{weight decay}
\end{subfigure}\hfill\begin{subfigure}{0.33\linewidth}
\includegraphics[width=\textwidth]{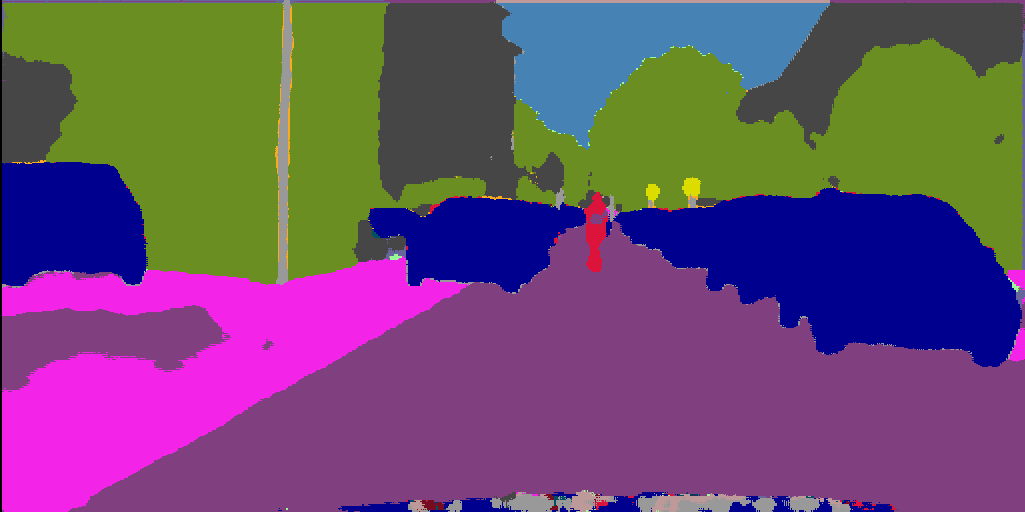}
\caption{weight decay + function norm}
\end{subfigure}
\caption{ENet outputs, after training on 500 samples of Cityscapes, without (b) and with (c) weighted function norm regularization (standard Cityscape color palette -- black regions are \emph{unlabelled} and not discounted in the evaluation.~\label{fig:enet_outputs}}
\end{figure}

In our experiments, we were not able to observe an improvement over the baseline in the same setting with a state-of-the-art semi-supervised method, mean-teacher \citep{tarvainen2017weight}. We therefore believe the observed effect to be attributed to the effect of the regularization. The impact of semi-supervision in the regime of such high resolution images for segmentation is however largely unknown and it is possible that a more thorough exploration of unsupervised methods would lead to a better usage of the unlabeled data.

\section{Discussion and Conclusions}

Regularization in deep neural networks has been challenging, and the most commonly applied frameworks only indirectly penalize meaningful measures of function complexity.  It appears that the better understanding of regularization and generalization in more classically considered function classes, such as linear functions and RKHSs, is due to the well behaved and convex nature of the function class and regularizers.  By contrast DNNs define poorly understood non-convex function sets.  Existing regularization strategies have not been shown to penalize a norm of the function.  We have shown here for the first time that norm computation in a low fixed depth neural network is NP-hard, elucidating some of the challenges of working with DNN function classes.  This negative result motivates the use of stochastic approximations to weighted norm computation, which is readily compatible with stochastic gradient descent optimization strategies.  We have developed gene backpropagation algorithms for weighted $L_2$ norms, and have demonstrated consistent improvement in performance over the most popular regularization strategies.
We  empirically validated the expected effect of the employed regularizer on generalization with experiments on the Oxford Flowers dataset, the MNIST image classification problem, and the training of ENet on Cityscapes. %ISBR brain segmentation task. 
We will make source code available at the time of publication.

%As a conclusion, the proposed regularizer is not only able to considerably improve the performance of neural networks in the small sample regime, but also has a behavior that is well explained by classical learning theory.

%\iffalse
\section*{Acknowledgments}
This work is funded by Internal Funds KU Leuven, FP7-MC-CIG 334380, the Research Foundation - Flanders (FWO) through project
number G0A2716N, and an Amazon Research Award.
%\fi

\vskip 0.4in
\bibliographystyle{plainnat}
\bibliography{biblio}

\clearpage
\appendix
\noindent{\huge\bfseries Function Norms and Regularization in Deep Neural Networks: Supplementary Material\par}

\vspace{3em}
\emph{In this Supplementary Material, Section~\ref{Suppl:nphard} details our NP-hardness proof of function norm computation for DNN. Section B. gives additional insight in the fact that weight decay does not define a function norm. 
Section C details our proof of our generalization bound for the L2 weighted function norm. Section~D gives some details on the VAE architecture used. Finally, Section~E gives additional results concerning Sobolev function norms.}

\section{NP-hardness of DNN \texorpdfstring{$L_2$}{L2} function norm~\label{Suppl:nphard}} 
We divide the proof of \thref{thm:NormNPhardConstruction} in the two following subsections. In Section~\ref{sec:defnp}, we introduce the necessary functions in order to build our constructive proof. Section~\ref{sec:proofnp} gives the proof of~\ref{thm:NormNPhardConstruction}, while Section~\ref{sec:AppendixOutputORblocks} demonstrate some technical property needed for one of the definitions in~\ref{sec:defnp}.

\subsection{Definitions\label{sec:defnp}}

\begin{definition} \thlabel{def:F0}
For 
a fixed $\varepsilon<0.5$, we define $f_0 : \mathbb{R} \rightarrow [0,1]$ as
\begin{equation}
f_0(x) = \varepsilon^{-1}[\max(0, x+\varepsilon) - 2\max(0,x) + \max(0, x- \varepsilon)],
\end{equation}
and 
\begin{equation}
f_1(x) = f_0(x-1).
\end{equation}
These functions place some non-zero values in the $\varepsilon$ neighborhood of $x = 0$ and $x = 1$, respectively, and zero elsewhere.  Furthermore, $f_0(0) = 1$ and $f_1(1) = 1$ (see Figure~\ref{fig:F0}).
\end{definition}
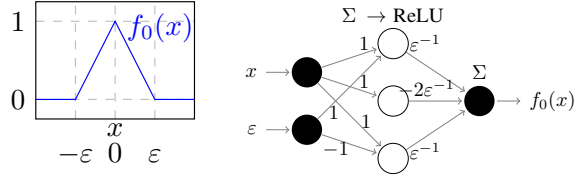
\begin{figure}[hb]
\centering
\resizebox{!}{2.3cm}{
\begin{tikzpicture}
\newcommand\widthplotaxis{3}
\newcommand\distepsilon{1.5}
\begin{axis}[
    clip=false,
    tick style={draw=none},
    width=3.5cm,height=3cm,
    xlabel={$x$},
    x label style={at={(axis description cs:0.5,0.05)},anchor=north},
    xmin=-\widthplotaxis, xmax=\widthplotaxis,
    ymin=-0.25, ymax=1.25,
    xtick={-\distepsilon,0,\distepsilon},
    xticklabels={$-\varepsilon$, $0$, $\varepsilon$},
    xticklabel style={text height=2.5ex},
%     extra x ticks={2.5},
%     extra x tick labels={$x$},
%     % extra x ticklabel style={text height=1.5ex},
%     extra x tick style={grid=none, draw=none,
%       tick label style={text height=1.5ex, yshift=-1mm}},
    ytick={0, 1},
    % legend pos=north east,
    xmajorgrids=true,
    ymajorgrids=true,
    grid style=dashed,
]

\addplot[
    color=blue,
    ]
    coordinates {
    (-\widthplotaxis,0)(-\distepsilon,0)(0,1)(\distepsilon,0)(\widthplotaxis,0)
    };
%    \legend{$f_0(x)$}
\node at (axis cs:1.7,0.6) [anchor=south, color=blue] {$f_0(x)$};
\end{axis}
\end{tikzpicture}
}
\resizebox{!}{2.3cm}{
\def\layersep{1.5cm}
\colorlet{mygreen}{black!50}
\begin{tikzpicture}[shorten >=1pt,->,draw=black!50, node distance=\layersep]
    \tikzstyle{every pin edge}=[<-,shorten <=1pt]
    \tikzstyle{neuron}=[circle,fill=black,minimum size=15pt,inner sep=0pt]
    \tikzstyle{input neuron}=[neuron];
    \tikzstyle{output neuron}=[neuron];
    \tikzstyle{hidden neuron}=[neuron, fill=none, draw=black];
    \tikzstyle{annot} = [text width=8em, text centered]

    % Draw the input layer nodes
    
     \node[input neuron, pin=left:$x$  ] (I-1) at (0,-1) {};
     \node[input neuron, pin=left:$\varepsilon$ ] (I-2) at (0,-2) {};

    % Draw the hidden layer nodes
    \foreach \name / \y in {1,...,3}
        \path[yshift=0.5cm]
            node[hidden neuron] (H-\name) at (\layersep,-\y cm) {};

    % Draw the output layer node
    \node[output neuron,pin={[pin edge={->}]right:$f_0(x)$}, right of=H-2] (O) {};

    % Connect every node in the input layer with every node in the
    % hidden layer.
    \foreach \source in {1,...,1}
        \foreach \dest in {1,...,3}
            \path (I-\source) edge node[near end, above, yshift=-2pt]{$1$} (H-\dest);
    \path (I-2) edge[draw=mygreen] node[near start, below, yshift=2pt]{$1$} (H-1);
%    \path (I-2) edge[draw=red!50] node[below]{\textcolor{red!50}{0}} (H-2);
    \path (I-2) edge[draw=mygreen] node[near start, below, yshift=2pt]{$-1$} (H-3);

    % Connect every node in the hidden layer with the output layer  
        
        \path (H-1) edge node[near start, above, yshift=2pt, xshift=2pt]{$\varepsilon^{-1}$}(O);
        \path (H-2) edge node[near start, above, yshift=-2pt, xshift=2pt]{$-2\varepsilon^{-1}$}(O);
        \path (H-3) edge node[near start, below, yshift=2pt, xshift=2pt]{$\varepsilon^{-1}$}(O);

   % Annotate the layers
    \node[annot,above of=H-1, node distance=.5cm] (hl) {$\Sigma \rightarrow$ ReLU};
    \node[annot,above of=O, node distance=.5cm] {$\Sigma$};
\end{tikzpicture}
}
\caption{Plot of function $f_0$, and network computing this function.}
\label{fig:F0}
\end{figure}

A sentence in $3$-conjunctive normal form (3-CNF) consists of the conjunction of $c$ clauses.  Each clause is a disjunction over $3$ literals, a literal being either a logical variable or its negation.  For the construction of our network, each variable will be identified with a dimension of our input $x \in \mathbb{R}^p$, we denote each of $c$ clauses in the conjunction $b_j$ for $1\leq j \leq c$, and each literal $l_{j_k}$ for $1 \leq k \leq 3$ will be associated with $f_0(x_i)$ if the literal is a negation of the $i$th variable or $f_1(x_i)$ if the literal is not a negation.  Note that each variable can appear in multiple clauses with or without negation, and therefore the indexing of literals is distinct from the indexing of variables.

\begin{definition}
We define the function $f_\wedge : [0,1]^c \rightarrow [0,1]$  as 
%\begin{equation}\label{def:and}
$f_\wedge(z) = f_0\left(\sum_{i=1}^{c} z_i - c\right),$ 
%\end{equation}
%where $f_\wedge(z) = 1$ for $z = \mathbf{1}$, %
with $c$ such that $f_\wedge(\mathbf{1}) = 1$ -- for $\mathbf{1}$ a vector of ones.
\end{definition}

\begin{definition}
We define the function $f_\lor : [0,1]^3 \rightarrow [0, 1]$ as 
\begin{equation}
\sum_{j=1}^3 f_0\left(\sum_{i=1}^3 z_i - j
\right). 
\end{equation}
For a proof that $f_\lor$ has values in $[0,1]$, see \thref{thm:NPhardnessLemmaOutputORblocksin01} in Section~\ref{sec:AppendixOutputORblocks} below.
% \begin{align}\label{def:or}
% f_\lor(z) =& \sum_{j=1}^3 f_0\left(\sum_{i=1}^3 z_i - j
% \right)  .
% \end{align}
\end{definition}

In order to ensure our network defines a function with finite measure, we may use the following function to truncate values outside the unit cube.
\begin{definition}\thlabel{def:SATconstructionTruncationFunction}
$f_{T}(x) = \| x \|_1 \cdot (1+ \operatorname{dim}(x))^{-1}$.
\end{definition}

\subsection{Proof of \texorpdfstring{\thref{thm:NormNPhardConstruction}}{ Lemma}\label{sec:proofnp}}

%This function is constructed such that it has a value greater than one only strictly outside the unit cube, and is straightforward to implement in a shallow network with ReLU activations. For bounded $f$, $\max\left( f(x) - f_{T}(x), 0\right)$ will truncate all values of $f$ outside a region of finite support guaranteeing finite measure.
\begin{proof}[Proof of~\thref{thm:NormNPhardConstruction}]
We construct the three hidden layers of the network as follows.
\begin{enumerate*}[label=(\roman*)]
\item In the first layer, we compute $f_0(x_i)$ for the literals containing negations and $f_1(x_i$) for the literals without negation. These operators introduce one hidden layer of at most $6 p$ nodes.
\item The second layer computes the clauses of three literals 
using the function $f_\lor$. This operator 
introduces one hidden layer with a number of nodes linear in the number of clauses in $\mathcal{B}$. 
\item Finally, each of the outputs of $f_\lor$ are concatenated into a vector and passed to the function $f_\wedge$. This operator requires one application $f_0$ and thus introduces one hidden layer with 
a constant number of nodes. 
\end{enumerate*}

Let $f_{\mathcal{B}}$ be the function coded by this network.  
By optionally adding an additional layer implementing the truncation in \thref{def:SATconstructionTruncationFunction} we can guarantee that the resulting function has finite $L_2$ norm. It remains to show that the norm of $f_{\mathcal{B}}$ is strictly positive if and only if $\mathcal{B}$ is satisfiable.

If $\mathcal{B}$ is satisfiable, let $x\in \{0, 1\}^p$ be a satisfying assignment of $\mathcal{B}$; by construction $f_{\mathcal{B}}(x)$ is $1$, as all the clauses evaluate exactly to 1. $f_{\mathcal{B}}$ being continuous by composition of continuous functions, we conclude that $\|f_{\mathcal{B}} \|_2 > 0$.

Now suppose $\mathcal{B}$ not satisfiable. For a given clause $b_j$, consider the dimensions associated with the variables contained within this clause and label them $x_{j_1}$, $x_{j_2}$, and $x_{j_3}$.  Now, for all $2^3$ possible assignments of the variables, consider the $2^3$ polytopes defined by restricting each $x_{j_k}$ to be greater than or less than $0.5$.  Exactly one of those variable assignments will have $l_{j_1} \lor l_{j_2} \lor l_{j_3} = \operatorname{false}$.  The function value over the corresponding polytope must be zero.  This is because the output of the $j$th $f_{\lor}$ must be zero over this region by construction, and therefore the output of the $f_{\wedge}$ will also be zero as the summation of all the $f_{\lor}$ outputs will be at most $c-1$.  For each of the $2^p$ assignments of the Boolean variables at least one clause will guarantee that $f_{\mathcal{B}}(x) = 0$ for all $x$ in the corresponding polytope, as the sentence is assumed to be unsatisfiable.  The union of all such polytopes is the entire space $\mathbb{R}^{p}$.  As $f_{\mathcal{B}}(x)=0$ everywhere, $\|f_{\mathcal{B}} \|_2 = 0$.
\end{proof}

\begin{corollary}
$\|\max\left( f_{\mathcal{B}} - f_{T}, 0\right) \|_2 > 0 \iff \mathcal{B}$ is satisfiable, and $\max\left( f_{\mathcal{B}} - f_{T}, 0\right)$ has finite measure for all $\mathcal{B}$.
\end{corollary}

\subsection{Output of \texorpdfstring{$OR$}{OR} blocks\label{sec:AppendixOutputORblocks}}
\begin{lemma}\thlabel{thm:NPhardnessLemmaOutputORblocksin01}
The output of all OR blocks in the construction of the network implementing a given SAT sentence has values in the range $[0,1]$.
\end{lemma}
\begin{proof}
Following the steps of Proposition 3, this function is defined for $X \in \mathbb{R}^3$ and:
\begin{align}
F(X) &= f_0(\sum_i f_1(X_i)- 1) + f_0(\sum_i f_1(X_i) - 2) \nonumber\\
&+ f_0(\sum_i f_1(X_i) - 3)
\end{align}
To compute the values of $F$ over $\mathbb{R}^3$, we consider two cases for every $X_i$: $X_i \in (1-\varepsilon, 1+\varepsilon)$ and $X_i \notin (1-\varepsilon, 1+\varepsilon)$.

\paragraph{Case 1: all $X_i \notin (1-\varepsilon, 1+\varepsilon)$:} In this case, we have $\sum_i f_1(X_i) = 0$. Therefore,$|\sum_i f_1(X_i) - k| > \varepsilon, \forall k \in \{1,2,3\}$, and $F(X) = 0.$

\paragraph{Case 2: only one $X_i \in (1-\varepsilon, 1+\varepsilon)$:}
Without loss of generality, we suppose that $X_1 \in (1-\varepsilon, 1+\varepsilon)$ and $X_{2,3} \notin (1-\varepsilon, 1+\varepsilon)$. Thus:
\begin{equation}
\sum_i f_1(X_i) = 1 - \frac{1}{\varepsilon} |X_1 - 1|.
\end{equation}
Thus, we have $\sum_i f_1(X_i) - 2 < -1$, $\sum_i f_1(X_i) - 3 < -2$, and $\sum_i f_1(X_i) - 1 < -\varepsilon \iff |X_1 - 1| > \varepsilon^2$. Therefore: 
\begin{equation}
F(X) = 
\begin{cases}
1 - \frac{1}{\varepsilon^2} |X_1-1|, \text{ for }0\leq |X_1-1| \leq \varepsilon^2 \\
0,\text{ otherwise}. 
\end{cases}
\end{equation}
\paragraph{Case 3: two $X_i \in (1-\varepsilon, 1+\varepsilon)$:} Suppose $X_{1,2} \in (1-\varepsilon, 1+\varepsilon)$. We have then:
\begin{equation}
\sum_i f_1(X_i) = 2 - \frac{1}{\varepsilon} |X_1 - 1| - \frac{1}{\varepsilon} |X_2 - 1|.
\end{equation}
Therefore: 
\begin{enumerate}
\item $\sum_i f_1(X_i) - 3 < -1,$
\item \begin{equation}
|\sum_i f_1(X_i) - 2| < \varepsilon \iff |X_1 - 1| + |X_2 - 1| < \varepsilon^2
\label{eq:SubRegion3_1}
\end{equation}
\item \begin{equation}
|\sum_i f_1(X_i) - 1| < \varepsilon \iff \varepsilon - \varepsilon^2 < |X_1 - 1| + |X_2 - 1| < \varepsilon+\varepsilon^2 
\label{eq:SubRegion3_2}
\end{equation}
\end{enumerate}
The resulting function values are then: 
\begin{equation}
F(X) = 
\begin{cases}
1 - \frac{1}{\varepsilon^2} |X_1-1|  - \frac{1}{\varepsilon^2} |X_2-1|, \text{ for } X_{1,2} \in \eqref{eq:SubRegion3_1} \\
1 - \frac{1}{\varepsilon} |1 - \frac{1}{\varepsilon} |X_1 - 1| - \frac{1}{\varepsilon} |X_2 - 1||, \text{ for } X_{1,2} \in \eqref{eq:SubRegion3_2} \\
0,\text{ otherwise}. 
\end{cases}
\end{equation}
As $\varepsilon < \frac{1}{2}$, the regions \eqref{eq:SubRegion3_1} and \eqref{eq:SubRegion3_2} do not overlap.
\paragraph{Case 4: all $X_i \in (1-\varepsilon, 1+\varepsilon)$:}
We have then:
\begin{equation}
\sum_i f_1(X_i) = 3 - \frac{1}{\varepsilon} |X_1 - 1| - \frac{1}{\varepsilon} |X_2 - 1| - \frac{1}{\varepsilon} |X_3 - 1|.
\end{equation}

Therefore 
\begin{enumerate}
\item \begin{align}
&|\sum_i f_1(X_i) - 2| < \varepsilon \nonumber\\ 
&\iff  |X_1 - 1| + |X_2 - 1| + |X_2 - 1| < \varepsilon^2
\label{eq:SubRegion4_1}
\end{align}
\item \begin{align}
&|\sum_i f_1(X_i) - 2| < \varepsilon \nonumber\\
&\iff\varepsilon - \varepsilon^2 < |X_1 - 1| + |X_2 - 1|+ |X_3 - 1|  < \varepsilon+\varepsilon^2 
 \label{eq:SubRegion4_2}
\end{align}
\item \begin{align}
&|\sum_i f_1(X_i) - 1| < \varepsilon \nonumber\\
&\iff 2\varepsilon - \varepsilon^2 < |X_1 - 1| + |X_2 - 1|+ |X_3 - 1|  < 2\varepsilon+\varepsilon^2 
 \label{eq:SubRegion4_3}
\end{align}
\end{enumerate}

The resulting function values are then 
\begin{equation}
F(X) = 
\begin{cases}
1 - \frac{1}{\varepsilon^2} \sum_i |X_i - 1|, \text{ for } X_{1,2,3} \in \eqref{eq:SubRegion4_1} \\
1 - \frac{1}{\varepsilon} |1 - \frac{1}{\varepsilon}\sum_i |X_i - 1||, \text{ for } X_{1,2,3} \in \eqref{eq:SubRegion4_2} \\
1 - \frac{1}{\varepsilon} |2 - \frac{1}{\varepsilon}\sum_i |X_i - 1||, \text{ for } X_{1,2,3} \in \eqref{eq:SubRegion4_3} \\
0,\text{ otherwise}. 
\end{cases}
\end{equation}
Again, as $\varepsilon < \frac{1}{2}$, the regions \eqref{eq:SubRegion4_1}, \eqref{eq:SubRegion4_2} and \eqref{eq:SubRegion4_3} do not overlap.
Finally,
\begin{equation}
\forall X\in  \mathbf{R}^3, 0 \leq F(X)\leq 1.
\end{equation}
\end{proof}

\section{Weight decay does not define a function norm}\label{sec:AppendixWD}

It is straightforward to see that weight decay, i.e.\ the norm of the weights of a network, does not define a norm of the function determined by the network.  Consider a layered network
\begin{equation}
f(x) = W_d \sigma( W_{d-1} \sigma( \dots \sigma( W_1 x) \dots )).
\end{equation}
where the non-linear activation function can be e.g.\ a ReLU.
The $L_2$ weight decay complexity measure is
\begin{equation}
\sum_{i=1}^{d} \|W_i \|_{\operatorname{Fro}}^2,
\end{equation}
where $\|\cdot \|_{\operatorname{Fro}}$ is the Frobenius norm.
A simple counter-example to show this cannot define a function norm is to set any of the matrices $W_j = \mathbf{0}$ and $f(x)=0$ for all $x$.  However $\sum_{i=1}^{d} \|W_i \|_{\operatorname{Fro}}^2$ can be set to an arbitrary value by changing the $W_i$ for $i\neq j$ although this does not change the underlying function.

\section{Proof of \texorpdfstring{\thref{th:generalization}}{the generalization bound}}\label{sec:AppendixProofGeneralizationBound}

\begin{proof}
In the following, $P$ is the marginal input distribution
\begin{equation}
P(x) = \int\! P(x,y) \,\mathrm{d}y
\end{equation}
We first suppose that $\mathcal{X}$ is bounded, and that all the activations of the network are continuous, so that any function $f$ represented by the network is continuous. Furthermore, if the magnitude of the weights are bounded (this condition will be subsequently relaxed), without further control we know that:
\begin{equation}
\exists L>0, \forall f \in \mathcal{H}, \forall x \in \mathcal{X}, \|f(x)\|_2 \le L,
\end{equation}
and supposing $\ell\ K$-Lipschitz continuous with respect to its first argument and under the $L_2$-norm, we have:
\begin{equation}
\forall x \in \mathcal{X}, |\ell(f(x),y) - \ell(0,y)| \le KL,
\end{equation}
and 
\begin{equation}
|\ell(f(x),y)| \le KL + |\ell(0,y)|.
\end{equation}
If we suppose $\mathcal{Y}$ bounded as well, then:
\begin{equation}
\exists C > 0, \forall (x,y) \in \mathcal{X}\times\mathcal{Y}, |\ell(f(x),y)| \le KL + C.
\end{equation}
Under these assumptions, using the Hoeffding inequality~\citep{hoeffding1963probability}, we have with probability at least 1-$\delta$:
\begin{equation}
\mathcal{R}(f) \le \hat{\mathcal{R}}(f)  + (KL + C)\sqrt{\frac{2\ln \frac{2}{\delta}}{n}}. %, \ \tilde{L} =  KL + C.
\end{equation}
When $n$ is large, this inequality insures a control over the generalization error when applied to $f_*$. However, when $n$ is small, this control can be insufficient. We will show in the following that under the constraints described above, we can further bound the generalization error by replacing $KL + C$ with a term that we can control.

In the the sequel, we consider $f$ verifying the conditions~\eqref{eq:Conditions}, while releasing the boundedness of $\mathcal{X}$ and the weights of $f$. Using Chebyshev's inequlity, we have with probability at least 1-$\delta$:
\begin{equation}
\forall x \in \mathcal{X}, |\|f(x)\|_2 - \mathbb{E}_{\nu\sim P}[\|f(\nu)\|_2]| \le \frac{\sigma_{f,P}}{\sqrt{\delta}},\ \text{where } \sigma_{f,P}^2 = \operatorname{var}_{\nu\sim P}(\|f(\nu)\|_2), 
\end{equation}
and 
\begin{equation}
\|f(x)\|_2 \le \frac{\sigma_{f,P}}{\sqrt{\delta}} + \mathbb{E}_{\nu\sim P}[\|f(\nu)\|_2].
\end{equation}
We have on the right-hand side of this inequality
\begin{align}
 \mathbb{E}_{\nu\sim P}[\|f(\nu)\|_2] &=\int\! \|f(\nu)\|_2 P(\nu)\, \mathrm{d}\nu% \\
% &=  \int\! \|f(\nu)\|_2 \frac{P(\nu)}{Q(\nu)}Q(\nu)\, \mathrm{d}\nu \\
\le  \underbrace{
\left( \int\! \|f(\nu)\|_2^2 Q(\nu)\, \mathrm{d}\nu \right)^\frac{1}{2}
}_{\|f\|_{2,Q}}
\left( \int\! \frac{P(\nu)^2}{Q(\nu)^2}Q(\nu)\, \mathrm{d}\nu \right)^\frac{1}{2}%\\
%&
%\le \|f\|_{2,Q} \left( \int\! \frac{P(\nu)}{Q(\nu)} P(\nu)\, \mathrm{d}\nu \right)^\frac{1}{2}
\end{align} 
using the Cauchy-Schwartz inequality. 
Similarly, we can write
\begin{align}
\sigma_{f,P}^2 &\le \int\! \|f(\nu)\|_2^2 P(\nu)\, \mathrm{d}\nu %\\
 % &=  \int\! \|f(\nu)\|_2^2 \frac{P(\nu)}{Q(\nu)}Q(\nu)\, \mathrm{d}\nu \\
 %&
 \le \left( \int\! \|f(\nu)\|_2^4 Q(\nu)\, \mathrm{d}\nu \right)^\frac{1}{2}\left( \int\! \left(\frac{P(\nu)}{Q(\nu)}\right)^2Q(\nu)\, \mathrm{d}\nu \right)^\frac{1}{2}\\
 &= \left(\operatorname{var}_{z \sim Q}(\|f_*(z)\|_2^2) + \mathbb{E}_{z \sim Q}(\|f_*(z)\|_2^2)^2\right)^\frac{1}{2} \left( \int\! \frac{P(\nu)}{Q(\nu)}P(\nu)\, \mathrm{d}\nu \right)^\frac{1}{2}\\
 &\le (A+B) \left( \int\! \frac{P(\nu)}{Q(\nu)} P(\nu)\, \mathrm{d}\nu \right)^\frac{1}{2}
\end{align}
To summarize, denoting $\mathcal{D}_P(P\|Q) =  \int\! \frac{P(\nu)}{Q(\nu)} P(\nu)\, \mathrm{d}\nu$, we have with probability at least 1-$\delta$, for any $x \in \mathcal{X}$ and $f$ satisfying~\eqref{eq:Conditions}:
\begin{equation}
\|f(x)\|_2 \le \frac{ (A+B)^\frac{1}{2} \mathcal{D}_P(P\|Q)^\frac{1}{4}  }{\sqrt{\delta}} + A^\frac{1}{2} \mathcal{D}_P(P\|Q)^\frac{1}{2} 
\end{equation}
Therefore, with probability at least $(1-\delta)^2$, 
\begin{equation}
\mathcal{R}(f) \le \hat{\mathcal{R}}(f)  + %\tilde{L}(A,B,  \mathcal{D}_P(P\|Q))
\underbrace{\left(K\left[\frac{ (A+B)^\frac{1}{2} \mathcal{D}_P(P\|Q)^\frac{1}{4}  }{\sqrt{\delta}} + A^\frac{1}{2} \mathcal{D}_P(P\|Q)^\frac{1}{2}\right] + C\right)}_{=:\tilde{L}(A, B, \operatorname{D_{P}}(P\|Q))}
\sqrt{\frac{2\ln \frac{2}{\delta}}{N}}. 
\end{equation}
$C$ is fixed and depends only on the loss function (e.g.\ for the cross entropy loss,  $C$ is the logarithm of the number of classes).  We note that $\tilde{L}(A,B, \mathcal{D}_{P}(P\|Q))$ is continuous and increasing in its arguments which finishes the proof.
\end{proof}

\section{Variational autoencoders} \label{sec:AppendixVAE}

\begin{figure}
\centering
\includegraphics[width=0.3\textwidth]{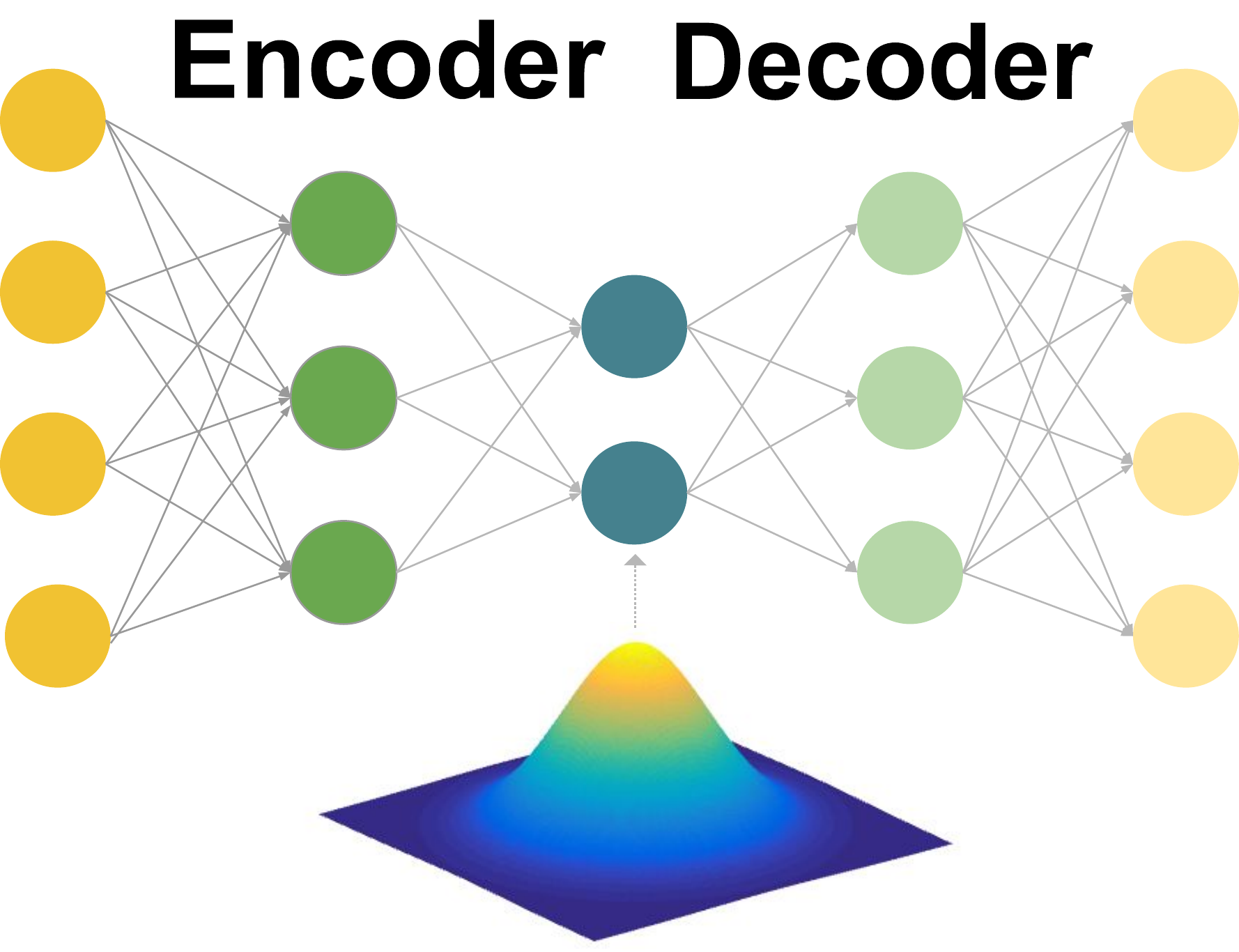}
\caption{VAE architecture\label{fig:VAE}}
\end{figure}
To generate samples for DNNs regularization, we choose to train VAEs on the training data. The chosen architecture is composed of two hidden layers for encoding and decoding. Figure~\ref{fig:VAE} displays such an architecture. For each of the datasets, the size of the hidden layers is set empirically to ensure convergence. The training is done with ADAM~\cite{kingma2014adam}. 
%The VAE of IBSR data has 512 and 10 nodes in the first and second hidden layer respectively, and is trained during 1000 epochs. 
As the latent space is mapped to a normal distribution, it is customary to generate the samples by reconstructing a normal noise. In order to have samples that are close to the data distribution but have a slightly broader support, we sample a normal variable with a higher variance. In our experiments, we multiply the variance by 2.

\section{Weighted Sobolev norms}

We may analogously consider a weighted Sobolev norm:
\begin{definition}[Weighted Sobolev norm]
\begin{align}
\|f\|_{H_2,Q}^2 =& \|f\|_{2,Q}^2 + \|\nabla_x f\|_{2,Q}^2 \\
=& \mathbb{E}_{x \sim Q}( \|f(x)\|_2^2 +  \|\nabla_x  f(x)\|_2^2 )
\end{align}
\end{definition}

\subsection{Computational complexity of weighted
\texorpdfstring{$L_2$}{L2} 
vs.\ Sobolev regularization}

We restrict our analysis of the computational complexity of the stochastic optimization to a single step as the convergence of stochastic gradient descent will depend on the variance of the stochastic updates, which in turn depends on the variance of $P$.

For the weighted $L_2$ norm, the complexity is simply a forward pass for the regularization samples in a given batch.  The gradient of the norm can be combined with the loss gradients into a single backward pass, and the net increase in computation is a single forward pass.

The picture is somewhat more complex for the Sobolev norm.  The first term is the same as the $L_2$ norm, but the second term penalizing the gradients introduces substantial additional computational complexity with computation of the exact gradient requiring a number of backpropagation iterations dependent on the dimensionality of the inputs.  We have found this to be prohibitively expensive in practice, and instead penalize a directional gradient in the direction of  $\varepsilon$, a random unitary vector that is resampled at each step to ensure convergence of stochastic gradient descent.

\subsection{Comparative performance of the Sobolev and $L2$ norm on MNIST}

Figure~\ref{fig:MNIST_sobo} displays the averaged curves and error bars on MNIST in a low-data regime for various regularization methods for the same network architecture and optimization hyperparameters. Comparisons are made between $L_2$, Sobolev, gradient (i.e.\ penalizing only the second term of the Sobolev norm), weight decay, dropout, and batch normalization.  In all cases, $L_2$ and Sobolev norms perform similarly, significantly outperforming the other methods.  

\begin{figure*}[ht]
\centering
\begin{subfigure}{0.3\textwidth}
\centering
\includegraphics[width=\columnwidth]{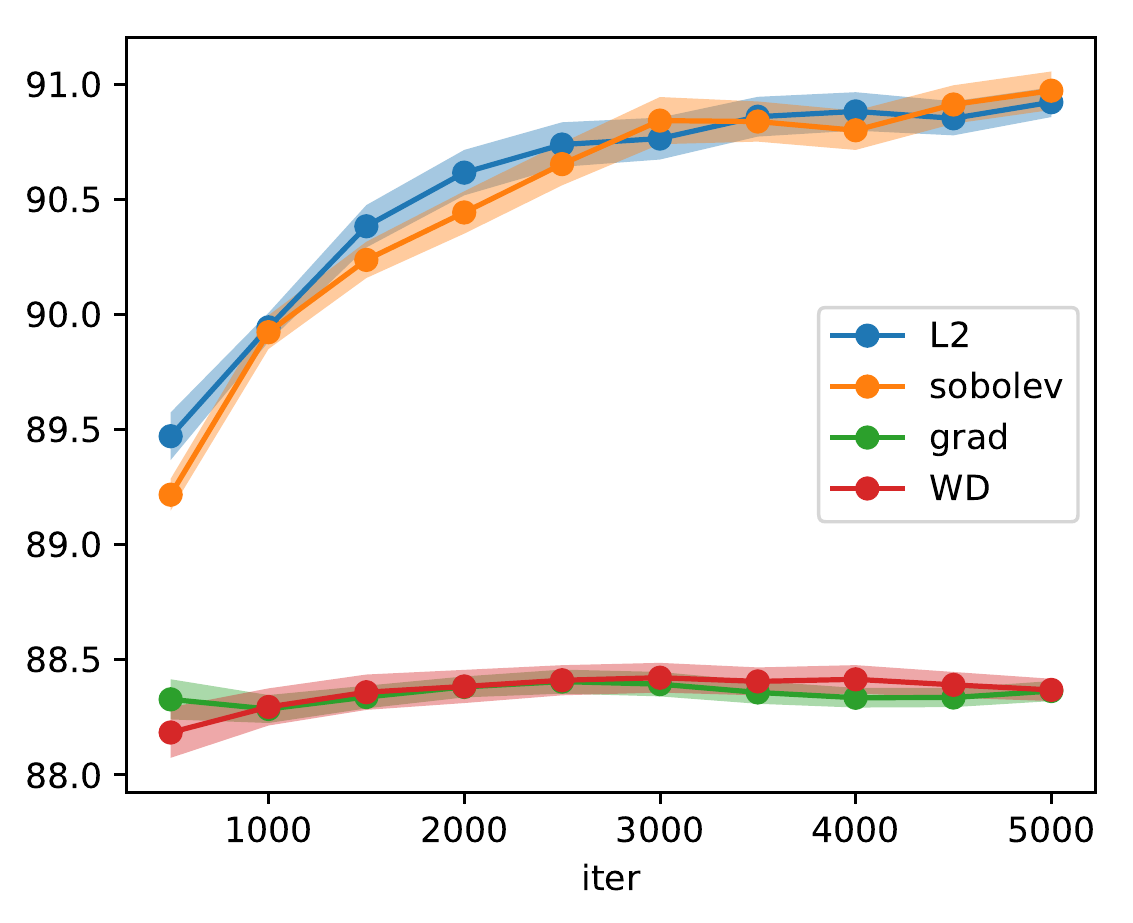}
\caption{Weighted norms vs.\ weight decay - no dropout}
\end{subfigure} \qquad \begin{subfigure}{0.3\textwidth}
\centering
\includegraphics[width=\columnwidth]{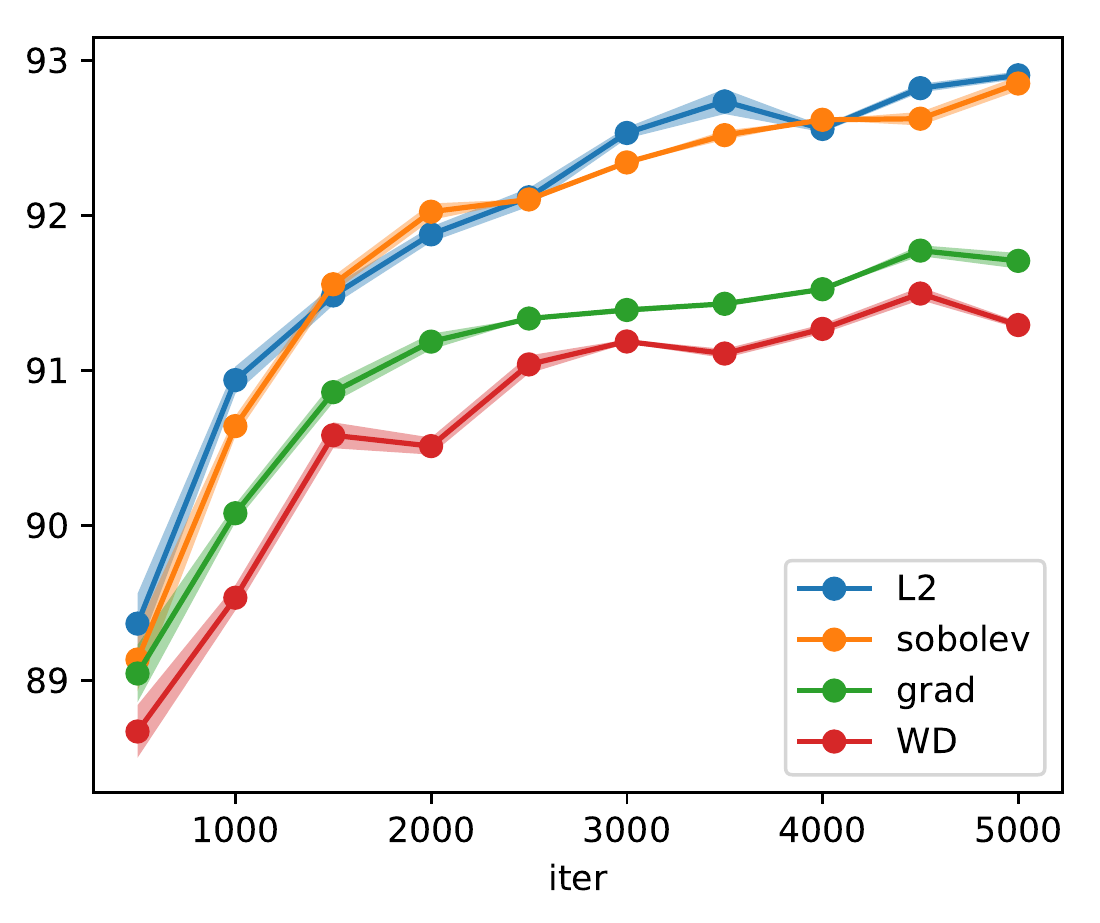}
\caption{Weighted norms vs.\ weight decay - with dropout}
\label{fig:MNIST_wd_do}
\end{subfigure}
\caption{Performance of weighted function norm regularization on MNIST in a low sample regime. In (a), we compare the regularizations when used without dropout. In (b), we compare them when used with dropout. The performance is averaged over 10 trials, training on different subsets of 300 samples, with a batch size for the regularization equal to 10 time the training batch, and a regularization parameter of 0.01. The regularization samples are sampled using a variational autoencoder.}
\label{fig:MNIST_sobo}
\end{figure*}

\end{document}